\documentclass[runningheads]{llncs}
\usepackage[misc]{ifsym}
\usepackage[T1]{fontenc}
\usepackage{times}
\usepackage{soul}
\usepackage{url}
\usepackage[hidelinks]{hyperref}
\usepackage[utf8]{inputenc}
\usepackage[small]{caption}
\usepackage{graphicx}
\usepackage{amsmath}
\usepackage{booktabs}
\usepackage[lined,boxed,commentsnumbered,linesnumbered,ruled]{algorithm2e}
\usepackage{multirow}
\usepackage{multicol}
\usepackage{threeparttable}
\usepackage{amssymb}

\usepackage{graphicx}
\usepackage{xspace}
\newcommand{\method}{\texttt{MoTiAC}\xspace}
\newcommand{\calI}{\mathcal{I}}
\newcommand{\calA}{\mathcal{A}}

\begin{document}
\newcommand*\samethanks[1][\value{footnote}]{\footnotemark[#1]}
\title{Multi-Objective Actor-Critics for Real-Time Bidding in Display Advertising\thanks{This work was supported by the National Key R\&D Program of China [2020YFB1707903]; the National Natural Science Foundation of China [61872238, 61972254]; Shanghai Municipal Science and Technology Major Project [2021SHZDZX0102]; the CCF-Tencent Open Fund [RAGR20200105] and the Tencent Marketing Solution Rhino-Bird Focused Research Program [FR202001]. Chaoqi Yang completed this work when he was an undergraduate student at Shanghai Jiao Tong University. Xiaofeng Gao is the corresponding author.}}
\titlerunning{Multi-Objective Actor-Critics for RTB in Display Advertising} 
\toctitle{Multi-Objective Actor-Critics for RTB in Display Advertising}


\author{Haolin Zhou\inst{1}\thanks{Equal contribution.} \and
	Chaoqi Yang\inst{2}\samethanks \and
	Xiaofeng Gao\inst{1}(\Letter) \and Qiong Chen\inst{3} \and Gongshen Liu\inst{4} \and Guihai Chen\inst{1}}

\institute{MoE Key Lab of Artificial Intelligence,\\
	Department of Computer Science and Engineering,\\
	Shanghai Jiao Tong University, Shanghai, China\\
	\and
	University of Illinois at Urbana-Champaign, Illinois, USA\\
	\and
	Tencent Ads, Beijing, China \\
	\and 
	 School of Electronic Information and Electrical Engineering,\\
	 Shanghai Jiao Tong University, Shanghai, China\\
	\email{\{koziello, lgshen\}@sjtu.edu.cn}, \email{	chaoqiy2@illinois.edu}\\
	\email{\{gao-xf, gchen\}@cs.sjtu.edu.cn}, \email{evechen@tencent.com}
}

\authorrunning{H. Zhou, C. Yang, X. Gao, Q. Chen, G. Liu, and G. Chen}
\tocauthor{Haolin~Zhou, Chaoqi~Yang, Xiaofeng~Gao, Gongsheng~Liu, Guihai~Chen}
\maketitle              % typeset the header of the contribution
\begin{abstract}
  Online Real-Time Bidding (RTB) is a complex auction game among which advertisers struggle to bid for ad impressions when a user request occurs. Considering display cost, Return on Investment (ROI), and other influential Key Performance Indicators (KPIs), large ad platforms try to balance the trade-off among various goals in dynamics. To address the challenge, we propose a \textbf{M}ulti-\textbf{O}bjec\textbf{Ti}ve \textbf{A}ctor-\textbf{C}ritics algorithm based on reinforcement learning (RL), named \method, for the problem of bidding optimization with various goals. In \method, objective-specific agents update the global network asynchronously with different goals and perspectives, leading to a robust bidding policy. Unlike previous RL models, the proposed \method can simultaneously fulfill multi-objective tasks in complicated bidding environments. In addition, we mathematically prove that our model will converge to Pareto optimality. Finally, experiments on a large-scale real-world commercial dataset from Tencent verify the effectiveness of \method versus a set of recent approaches.
\end{abstract}

\keywords{Real-time Bidding \and Reinforcement Learning \and Display Advertising \and Multiple Objectives} 

\section{Introduction}
The rapid development of the Internet and smart devices has created a decent
environment for the advertisement industry. As a result, real-time bidding (RTB) has gained continuous attention in the past few decades~\cite{yuan2013real}. A typical RTB setup consists of publishers, supply-side platforms (SSP), data management platforms (DMP), ad exchange (ADX), and demand-side platforms (DSP). When an online browsing activity triggers an ad request in one bidding round, the SSP sends this request to the DSP through the ADX, where eligible ads compete for the impression. The bidding agent, DSP, represents advertisers to come up with an optimal bid and transmits the bid back to the ADX (e.g., usually within less than 100ms~\cite{yuan2013real}), where the winner is selected to be displayed and charged by a generalized second price (GSP).

In the RTB system, \emph{bidding optimization} in DSP is regarded as the most critical problem~\cite{zhang2014optimal}. Unlike Sponsored Search (SS)~\cite{zhao2018deep}, where advertisers make keyword-level bidding decisions, DSP in the RTB setting needs to calculate the optimal impression-level bidding under the basis of user/customer data (e.g., income, occupation, purchase behavior, gender, etc.), target ad (e.g., content, click history, budget plan, etc.) and auction context (e.g., bidding history, time, etc.) in every single auction~\cite{zhang2014optimal}.

Thus, our work focuses on DSP, where \emph{bidding optimization} happens. In real-time bidding, two fundamental challenges need to be addressed. Firstly, the RTB environment is highly dynamic. In~\cite{wang2017ladder,zhang2014optimal,zhu2017optimized}, researchers make a strong assumption that the bidding process is stationary over time. However, the sequence of user queries (e.g., incurring impressions, clicks, or conversions) is time-dependent and mostly unpredictable~\cite{zhao2018deep}, where the outcome influences the next auction round. Traditional algorithms usually learn an independent predictor and conduct fixed optimization that amounts to a greedy strategy, often not leading to the optimal return~\cite{cai2017real}. Agents with reinforcement learning (RL) address the aforementioned challenge to some extent~\cite{jin2018real,lu,zhao2018deep}. RL-based methods can alleviate the instability by learning from immediate feedback and long-term reward. However, these methods are limited to either \emph{Revenue} or \emph{ROI}, which is only one part of the overall utility. In the problem of RTB, we assume that the utility is two-fold, as outlined: (i) the cumulative cost should be kept within the budget; (ii) the overall revenue should be maximized. Therefore, the second challenge is that the real-world RTB industry needs to consider multiple objectives, which are not adequately addressed in the existing literature.

To address the challenges mentioned above, we propose a \emph{Multi-Objective Actor-Critic} model, named \method. We generalize the popular asynchronous advantage actor-critic (A3C)~\cite{mnih2016asynchronous} reinforcement learning algorithm for multiple objectives in the RTB setting. Our model employs several local actor-critic networks with different objectives to interact with the same environment and then updates the global network asynchronously according to different reward signals. Instead of using a fixed linear combination of different objectives, \method can decide on adaptive weights over time according to how well the current situation conforms with the agent's prior. We evaluate our model on click data collected from the Tencent ad bidding system. The experimental results verify the effectiveness of our approach versus a set of baselines.

The contributions in this paper can be summarized as follows: 
\begin{itemize}
	\item We identify two critical challenges in RTB and are well motivated to use multi-objective RL as the solution.
	\item We propose a novel multi-objective actor-critic model \method for optimal bidding and prove the superiority of our model from the perspective of Pareto optimality.
	\item Extensive experiments on a real industrial dataset collected from the Tencent ad system show that \method achieves state-of-the-art performance. 
\end{itemize}

%%%%% Preliminaries %%%%%
\section{Preliminaries}
\subsection{Definition of oCPA and Bidding Process} \label{sec:pd}
In the online advertising scenario, there are three main ways of pricing. Cost-per-mille (CPM)~\cite{jin2018real} is the first standard, where revenue is proportional to \emph{impression}. Cost-per-click (CPC)~\cite{zhang2014optimal} is a performance-based model, i.e., only when users \emph{click} the ad can the platform get paid. In the cost-per-acquisition (CPA) model, the payment is attached to each \emph{conversion} event. Regardless of the pricing ways, ad platforms always try to maximize revenue while simultaneously maintaining the overall cost within the budget predefined by advertisers.

In this work, we focus on one pricing model that is currently used in Tencent online ad bidding systems, called optimized cost-per-acquisition (oCPA), in which \textbf{advertisers are supposed to set a target CPA price, denoted by $\text{CPA}_{\text{target}}$ for each conversion while the charge is based on each click}. The critical point for the bidding system is to make an optimal strategy to allocate overall impressions among ads properly, such that (i) the real click-based cost is close to the estimated cost calculated from $\text{CPA}_{\text{target}}$, specifically,
\begin{equation}
	\text{\#clicks}\times \text{CPC}_\text{next} \approx \text{\#conversions}\times \text{CPA}_{\text{target}},
	\label{eq:eq}
\end{equation}
where $\text{CPC}_\text{next}$ is the cost charged by the second highest price and $\text{CPA}_{\text{target}}$ is pre-defined for each conversion; (ii) more overall conversions. In the system, the goal of our bidding agent is to generate an optimal $\textbf{CPC}_{\textbf{bid}}$ price, adjusting the winner of each impression. We denote $\calI=\{1, 2, ..., n\}$ as bidding iterations, $\calA=\{ad_1, ad_2, ...\}$ as a set of all advertisements. For each iteration $i\in \calI$, $ad_j\in \calA$, our bidding agent will decide on a $\text{CPC}_{\text{bid}}^{(i,j)}$ to play auction. Then the ad with the highest $\text{CPC}_{\text{bid}}^{(i,j)}$ wins the impression and then receives possible $\text{\#clicks}^{(i,j)}$ (charged by $\text{CPC}_{\text{next}}^{(i,j)}$ per click) and $\text{\#conversions}^{(i,j)}$ based on user engagements. 

\subsection{Optimization Goals in Real-Time Bidding}\label{sec:aac}

On the one hand, when $\text{CPC}_{\text{bid}}$ is set higher, ads are more likely to win this impression to get clicks or later conversions, and vice versa. However, on the other hand, higher $\text{CPC}_{\text{bid}}$ means lower opportunities for other ad impressions. Therefore, to determine the appropriate bidding price, we define the two optimization objectives as follows:

\subsubsection{Objective 1: Minimize overall CPA.}
The first objective in RTB bidding problem is to allocate impression-level bids in every auction round, so that each ad will get reasonble opportunities for display and later get clicks or conversions, which makes $\textbf{CPA}_\textbf{real}$ close to $\textbf{CPA}_\textbf{target}$ pre-defined by the advertisers:
\begin{equation}
	\label{o1}
	\textbf{CPA}_\textbf{real}^{(j)} = \frac{\sum_{i\in \calI}\text{\#clicks}^{(i,j)}\times \text{CPC}_\text{next}^{(i,j)}}{\sum_{i\in \calI}\text{\#conversions}^{(i,j)}}
	, ~~~ \forall ad_j \in \calA.
\end{equation}
To achieve the goal of minimizing overall CPA, i.e., be in line with the original budget, a lower ratio between $\textbf{CPA}_\textbf{real}^{(j)}$ and $\textbf{CPA}_\textbf{target}^{(j)}$ is desired. Precisely, when the ratio is smaller than 1, the agent will receive a positive feedback. On the contrary, when the ratio is greater than 1, it means that the actual expenditure exceeds the budget and the agent will be punished by a negative reward. 

\subsubsection{Objective 2: Maximize conversions.}
The second objective is to enlarge conversions as much as possible under the condition of a reasonable $\text{CPA}_\text{real}$, so that platform can stay competitive and run a sustainable business:
\begin{equation}
	\label{o2}
	\textbf{\#conversions}^{(j)} = \sum_{i\in \calI}\text{\#conversions}^{(i,j)}
	, ~~~ \forall ad_j \in \calA,
\end{equation}
where $\textbf{\#conversions}^{(j)}$ is a cumulative value until the current bidding auction. Obviously, relatively high \#conversions will receive a positive reward. When the policy network gives fewer conversions, the agent will be punished with a negative reward.  

Note that in the real setting, optimization objectives used by advertising platforms can be adjusted based on actual business needs. In the implementation and evaluation of \method, we use \textbf{ROI} (Return on Investment) and \textbf{Revenue}, corresponding to the two objectives for optimization, i.e., minimizing overall CPA and maximizing the number of conversions. Their definition will be detailed in Sec.~\ref{exp:setup}.

%%%%% Methodology %%%%%
\section{Methodology}
As shown in Sec.~\ref{sec:pd}, the RTB problem is a multi-objective optimization problem. We need to control advertisers' budgets and make profitable decisions for the ad platform. Traditional RTB control policy or RL agent can hardly handle these challenges. In this work, we design \method to decouple the training procedure of multiple objectives into disentangled worker groups of actor-critics. We will elaborate on the technical details of \method in the following subsections. 

\subsection{Asynchronous Advantage Actor-Critic Model in RTB}\label{sec:ac}
An actor-critic reinforcement learning setting~\cite{konda2000actor} in our RTB scenario consists of:
\begin{itemize}
	\item \textbf{state $s$:} each state is composed of anonymous feature embeddings extracted from the user profile and bidding environment, indicating the current bidding state.
	\item \textbf{action $a$:} action is defined as the bidding price for each ad based on the input state. Instead of using discrete action space~\cite{wang2017ladder}, our model outputs a distribution so that action can be sampled based on probability.
	\item \textbf{reward $r$:} obviously, the reward is a feedback signal from the environment to evaluate how good the previous action is, which guides the RL agent towards a better policy. In our model, we design multiple rewards based on different optimization goals. Each actor-critic worker group deals with one type of reward from the environment and later achieves multiple objectives together.
	\item \textbf{policy $\pi_{\theta}(\cdot)$:} policy is represented as $\pi_{\theta}(a_t|s_t)$, which denotes the probability to take action $a_t$ under state $s_t$. In an actor-critic thread, actor works as a policy network, and critic stands for value function $V(s; \theta_v)$ of each state. The parameters are updated according to the experience reward obtained during the training process. 
\end{itemize}
\begin{figure}[t]
	\centering
	\captionsetup{labelfont=bf}
	\includegraphics[width=1\textwidth]{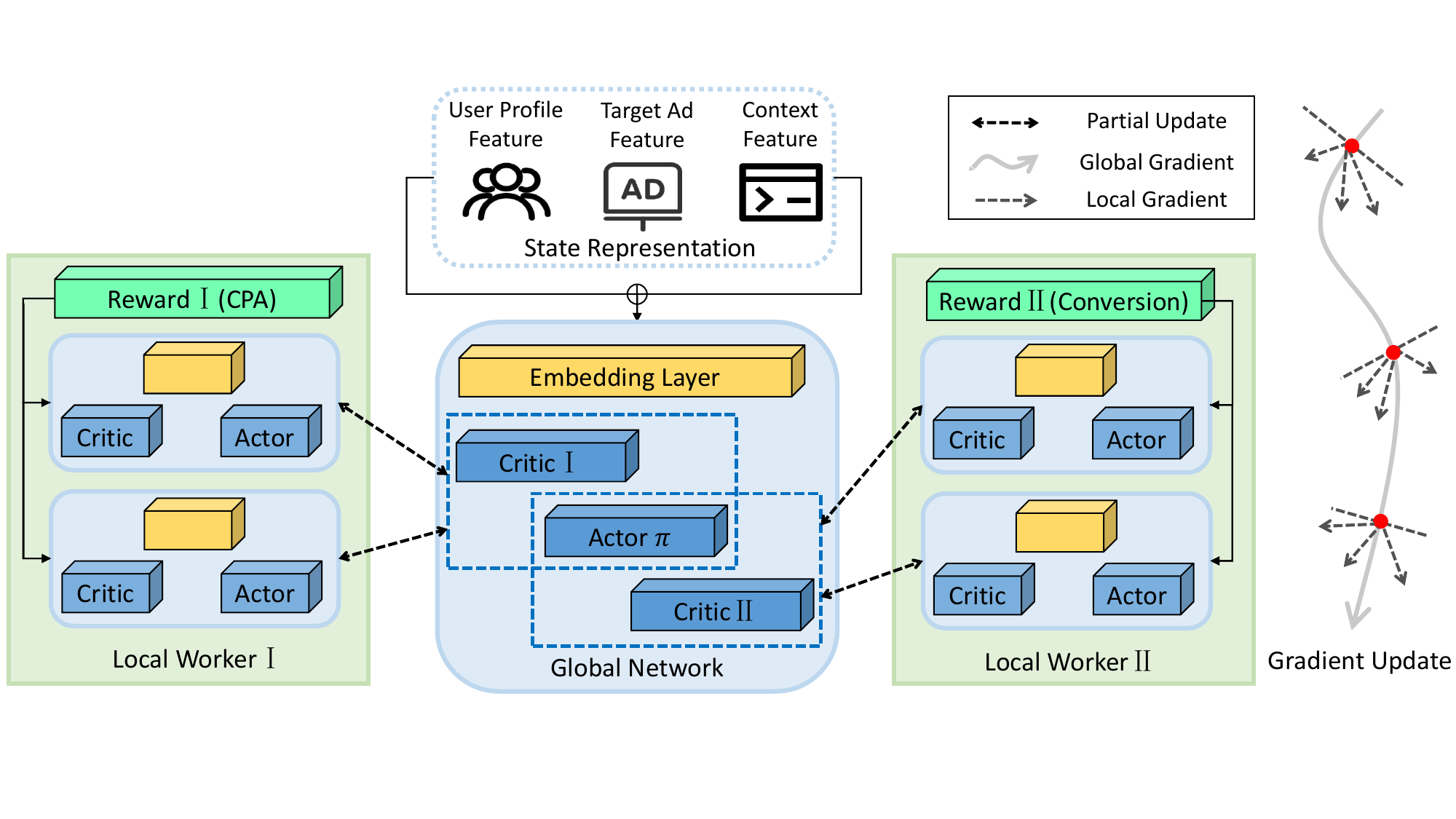}
	\caption{Framework of the proposed \method in RTB.}
	\label{fig:MoTiAC}
\end{figure}

For each policy $\pi_{\theta}$, we define the utility function as
\begin{equation}
U(\pi_{\theta}) = E_{\tau\sim{p_{{\theta}}(\tau)}} [R(\tau)],
\end{equation}
where $p_{\theta}(\tau)$ denotes the distribution of trajectories under policy $\pi_{\theta}$, and $R(\tau)$ is a return function over trajectory $\tau$, calculated by summing all the reward signals in the trajectory. The utility function is used to evaluate the quality of an action taken in a specific state. We also introduce value function from critic to reduce the varaition that may occur when updating parameters in real time. After collecting a number of tuples $(s_t,a_t,r_t)$ from each trajectory $\tau$, the policy network $\pi_{\theta}(\cdot)$ is updated by
\begin{align}\label{eq:sgdv}
	\theta &\leftarrow \theta + \eta_{actor}\sum_{t=1}^{T}(R(s_t)-V(s_t))\nabla_{\theta}\log\pi_{\theta}(a_t\mid s_t),
\end{align}
where $\eta_{actor}$ represents the learning rate of policy network, $T$ is a preset maximum step size in a trajectory, $R(s_t)=\sum_{n=t}^{T}\gamma^{n-t}r_n$ denotes the cumulative discounted reward, and $\gamma$ is a decaying factor. The critic network, $V(s; \theta_v)$, could also be updated by:
\begin{align}\label{eq:sgdv2}
	\theta_v \rightarrow \theta_v+\eta_{critic}\frac{\partial\left(R_{t}-V_{\theta_v}\left(s_{t}\right)\right)^{2}}{\partial \theta_v}, 
\end{align}
where $\eta_{critic}$ represents the learning rate of value function. 

\subsection{Adaptive Reward Partition}\label{sub:reward_partition}
In this subsection, we consider the general $K$-objective case, where $K$ is the total number of objectives. As stated in Sec.~\ref{sec:aac}, multiple objectives should be considered in modeling the RTB problem. One intuitive way~\cite{pasunuru2018multi} of handling multiple objectives is to integrate them into a single reward function linearly, and we call it \emph{Reward Combination}: (i) A linear combination of rewards is firstly computed, where $w_k$ quantifies the relative importance of the corresponding objective $R_k(\cdot)$:
\begin{equation}
	R(s)=\sum_{k=1}^K w_k\times R_k(s).
\end{equation}
(ii) A single-objective agent is then defined with the expected return equal to value function $V(s)$. However, a weighted combination is only valid when objectives do not compete~\cite{sener2018multi}. In the RTB setting, the relationship between objectives can be complicated, and they usually conflict on different sides. The intuitive combination might flatten the gradient for each objective, and thus the agent is likely to limit itself within a narrow boundary of search space. Besides, a predefined combination may not be flexible in the dynamic bidding environment. Overall, such a \emph{Reward Combination} method is unstable and inappropriate for the RTB problem, as we will show in the experiments.

\subsubsection{Reward Partition.}
We now propose the \emph{Reward Partition} scheme in \method. Specifically, we design reward for each objective and employ one group of actor-critic networks with the corresponding reward. There is one global network with an actor and multiple critics in our model. At the start of one iteration, each local network copies parameters from a global network. Afterward, local networks from each group will begin to explore based on their objective and apply weighted gradients to the actor and one of the critics (partial update) in the global network asynchronously, as shown in Fig.~\ref{fig:MoTiAC}. Formally, we denote the total utility and value function of the $k^{th}$ group ($k=1,\cdots,K$) as $U^k(\pi_{\theta})$ and $V_k(s;\theta_{v})$, respectively. Different from the original Eqn.~\eqref{eq:sgdv}, the parameter updating of policy network in one actor-critic group of \method is formulated as
\begin{align}\label{eq:sgd2}
%\theta &  \leftarrow \theta + \eta_{actor} w_k\nabla_{\theta}U^k(\pi_{\theta}), \\
\theta & \leftarrow \theta + \eta_{actor} w_k\sum_{t=1}^{T} (R_k(s_t)-V_k(s_t))\nabla_{\theta}\log\pi_{\theta}(a_t\mid s_t),
\end{align}
where $w_k$ is an objective-aware customized weight for optimization in range (0,1) and is tailored for each $ad_{j}\in A$. We can simply set $w_{k}$ as 
\begin{equation}
	w_k = \frac{R_{k}(s_t)-V_{k}(s_t)}{\sum_{l=1}^K (R_{l}(s_t)-V_l(s_t))}, 
\end{equation}
while dynamically adjusting the value of $w_{k}$ by giving higher learning weights to the local network that contributes more to the total reward. Motivated by Bayesian RL \cite{ghavamzadeh2015bayesian}, we can generalize the customized weight and parameterize $w_k$ by introducing a latent multinomial variable $\phi$ with $w_k = p(\phi=k|\tau)$ under trajectory $\tau$, named as \emph{agent's prior}. We set the initial prior as
\begin{equation}
p(\phi=k|\tau_0)=\frac1K,~~~\forall~k = 1, 2, \dots, K, 
\end{equation}
where $\tau_0$ indicates that the trajectory just begins. When $\tau_t$ is up to state $s_t$, i.e., $\tau_t=\left\{s_{1}, a_{1}, r_{1}, s_{2}, a_{2}, r_{2}, \ldots s_{t}\right\}$, we update the posterior by
\begin{align}\label{eq:phi}
p(\phi=k|\tau_t) = \frac{p(\tau_t|\phi=k)p(\phi=k)}{\sum_k p(\tau_t|\phi=k)p(\phi=k)},
\end{align}
where $p(\tau_t|\phi=k)$ tells how well the current trajectory agrees with the utility of objective $k$. Based on priority factor $w_k$, together with the strategy of running different exploration policies in different groups of workers, the overall changes being made to the global actor parameters $\theta$ are likely to be less correlated and more objective-specific in time, which means our model can make wide exploration and achieve a balance between multiple objectives with a global overview.

In addition, we present some analysis for the two reward aggregation methods in terms of parameters update and value function approximation. If we attach the weights of \emph{Reward Combination} to the gradients in \emph{Reward Partition}, the parameters updating strategy should be identical on average. For \emph{Reward Combination}, global shared actor parameters $\theta$ is updated by
\begin{equation*}
	\theta \leftarrow \theta + \eta_{actor}\sum_{t}\left( \left(\sum_{k=1}^{K}w_k\times R_k(s_t)-V_k(s_t)\right)\times \nabla_\theta\log\pi_\theta(a_t\mid s_t)\right),
\end{equation*}
while in \emph{Reward Partition}, the expected global gradient is given as
\begin{equation*}
	\theta \leftarrow \theta +
	\eta_{actor}\sum_{t}\left(\left(\sum_{k=1}^{K}(w_k\times R_k(s_t)-w_k\times V_k(s_t))\right)\times \nabla_\theta\log\pi_\theta(a_t\mid s_t)\right).
\end{equation*}
The difference between the two reward aggregating methods lies in the advantage part. Thus the effect of parameter updates heavily depends on how well and precisely the critic can learn from its reward. By learning in a decomposed manner, the proposed \emph{Reward Partition} advances the \emph{Reward Combination} by using easy-to-learn functions to approximate single rewards, thus yielding a better policy.

\subsection{Optimzation and Training Procedure}\label{sub:agg}
In the framework of \method, the policy network explores continuous action space and outputs action distribution for inference. Therefore, loss for a single actor-critic worker (objective-$k$) is gathered from actor $\theta$, critic $\theta_v$, and action distribution entropy $H$ to improve exploration by discouraging premature convergence to sub-optimal~\cite{mnih2016asynchronous},
\begin{equation}
	L_{\theta, \theta_v} = \eta_{actor} E[R(\tau)] + \eta_{critic} \sum_{s_{t} \in \tau}\left\|V_{\theta_{v}}\left(s_{t}\right)-R(s_t)\right\|^{2} + \beta \sum_{s_{t} \in \tau} H(\pi(s_t)),
\end{equation}
where $\beta$ represents the strength of entropy regularization. 

After one iteration (e.g., 10-minute bidding simulation), we compute gradients for each actor-critic network and push the weighted gradients to the global network. With multiple actor-learners applying online updates in parallel, the global network could explore to achieve a robust balance between multiple objectives. The training procedure of \method is shown in Algorithm~\ref{algo:MoTiAC}. 

\begin{algorithm}[t]\small
	\SetAlgoLined
	// Assume global shared parameters $\theta$ and $\theta_v$\;
	// Assume objective-specific parameters $\theta_k'$ and $\theta_{v,k}'$, $k\in \{1, 2, \dots, K\}$\;
	Initialize step counter $t \leftarrow 1$; epoch $T$; discounted factor $\gamma$\;
	\While{$t<T_{max}$}{
		Reset gradients: $d\theta \leftarrow 0$ and $d\theta_v \leftarrow 0$ \;
		Synchronize specific parameters $\theta_k'=\theta$ and $\theta_{v,k}' = \theta_{v}$\;
		Get state $s_t$ extracted from user profile features and bidding environment\;
		// Assume ad set $\calA=\{ad_1, ad_2, ...\}$\;
		\For{$ad_j \in \calA$}{
			\Repeat{terminal state}{
				Determine bidding price $a_t$ according to policy $\pi(a_t\mid s_t;\theta_k')$\;
				Receive reward $r_t$ w.r.t \emph{objective k}\;
				Reach new state $s_{t+1}$\; 
				$t\leftarrow t+1$\;}
			\For{$n \in \{t-1, ..., 1\}$}{
				$r_n\leftarrow r_n + \gamma\times r_{n+1}$\;
				// Accumulative gradient w.r.t $\theta'_k$\;
				$d\theta'_k \leftarrow d\theta'_k + \eta_{actor}\sum(r_n-V(s_n;\theta'_{v,k}))\nabla_{\theta_k'}\log\pi(a_n|s_n)+\beta\sum \nabla_{\theta_k^{\prime}} H\left(\pi\left(a_n|s_n\right)\right)$\;
				// Accumulative gradient w.r.t $\theta'_{v,k}$\;
				$d\theta'_{v,k} \leftarrow d\theta'_{v,k} + \eta_{critic}\sum\partial\Vert r_n-V(s_n;\theta'_{v,k})\Vert^2 / \partial \theta'_{v,k}$\;
			}
			// Asynchronously update $\theta$ and $\theta_{v}$ with $d\theta_k'$ and $d\theta_{v,k}'$\;
			// Compute $w_k = p(\phi=k|\tau)$ w.r.t objective $k$\;
			$\theta \leftarrow \theta + w_k\times d\theta_k'$ and $\theta_{v} \leftarrow \theta_{v} + w_k\times d\theta_{v,k}'$\;
		}
	}
	\caption{Training for each actor-critic thread in \method}
	\label{algo:MoTiAC}
\end{algorithm}

\subsection{Convergence Analysis of \method}\label{sec:wac}
In this section, we use a toy demonstration to provide insights into the convergence property for the proposed \method. As illustrated in Fig.~\ref{fig:converence}, the solid black line is the gradient contour of \emph{objective 1}, and the black dash line is for \emph{objective 2}. The yellow area within their intersection is the area of the optimal strategy, where both advertisers and publishers satisfy with their benefits. Due to the highly dynamic environment of RTB \cite{cai2017real}, the optimal bidding strategy will change dramatically.
\begin{figure}
	\centering
	\captionsetup{labelfont=bf}
	\includegraphics[width=3.3in]{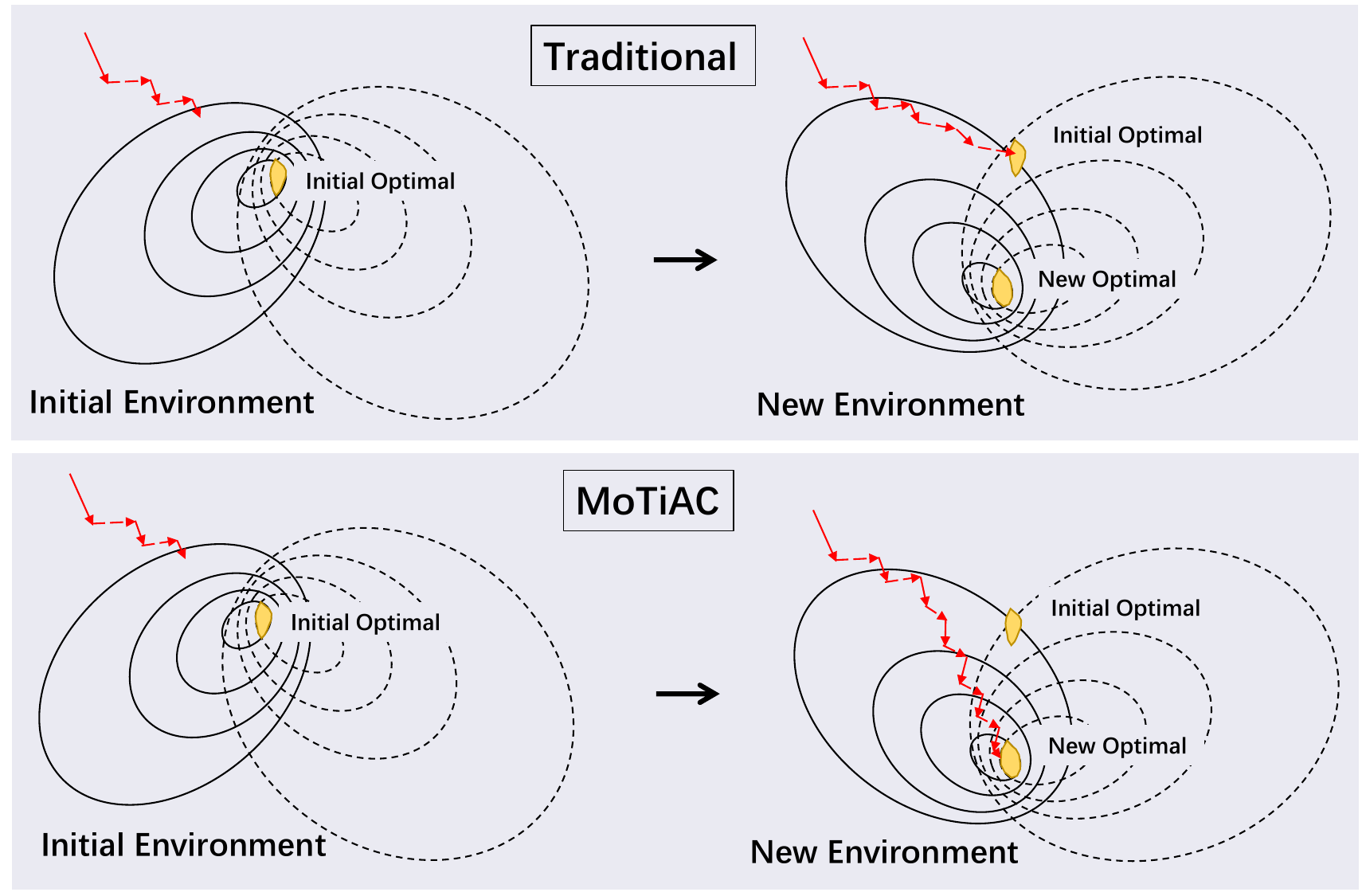}
	\caption{Convergence illustration of \method.}
	\label{fig:converence}
\end{figure}

Traditionally in a multi-objective setting, when people use linear combinations or other more complex transformations~\cite{lizotte2010efficient}, like policy votes~\cite{shelton2001balancing} of reward functions. They implicitly assume that the optimal solution is fixed, as shown in the upper part of Fig.~\ref{fig:converence}. Consequently, their models can only learn the \emph{initial optimal} and fail to characterize the dynamics. However, according to the dynamic environment in RTB, our \method adjusts the gradient w.r.t each possible situation towards a new optimal based on each objective separately and will easily be competent for real-world instability. Each gradient w.r.t the objectives forces the agent closer to the optimal for compensation rather than conflicts. Finally, the agent would reach the area of \emph{new optimal} and tunes its position in the micro-level, called convergence.

We further prove that the global policy will converge to the Pareto optimality between these objectives. The utility expectation of the objective $k$ is denoted as $E[U^k(\pi_{\theta})]$. We begin the analysis with Theorem~\ref{thm:pareto}~\cite{Pareto},

\begin{theorem}\label{thm:pareto}
	(Pareto Optimality). If $\pi^*$ is a Pareto optimal policy, then for any other policy $\pi$, one can at least find
	one k, so that $0<k\leq K$ and,
	\begin{equation}
		E[U^k(\pi^*)] \geq E[U^k(\pi)].
	\end{equation}
\end{theorem}
The multi-objective setting assumes that the possible policy set $\Pi$ spans a convex space ($K$-simplices). The optimal policy of any affine interpolation of objective utility will be also optimal~\cite{critch2017toward}. We restate in Theorem~\ref{thm:pareto2} by only considering the non-negative region.
\begin{theorem}\label{thm:pareto2}
	$\pi^*$ is Pareto optimal iff there exits $\{l_k>0: \sum_kl_k=1\}$ such that,
	\begin{equation}
		\pi^* \in \underset{\pi}{\operatorname{arg\,max}}~\left[\sum_kl_kE[U^k(\pi)]\right]. \label{eq:pareto2}
	\end{equation}
\end{theorem}
\begin{proof}
	We derive the gradient by aggregating Eqn.~\eqref{eq:sgd2} as,
	\begin{equation}
		\begin{aligned}
			\nabla &= \sum_{\tau_t}\sum_{k}p(\phi=k|\tau_t)\nabla_{\theta}U^k(\tau_t; \pi_{\theta})
			\propto \sum_{k}p(\phi=k)\sum_{\tau_t}p(\tau_t|\phi=k)\nabla_{\theta}U^k(\tau_t; \pi_{\theta}) \\
			& =\sum_{k}p(\phi=k)\nabla_{\theta}E_{\tau_t}[U^k(\tau_t; \pi_{\theta})]
			=\nabla_{\theta}\left[ \sum_{k}p(\phi=k)E_{\tau_t}[U^k(\tau_t; \pi_{\theta})]\right].
		\end{aligned}
	\end{equation}
	By making $l_k=p(\phi=k)$ (Note that $\sum_kp(\phi=k)=1$), we find that the overall gradient conform with the definition of Pareto optimality in Eqn.~\eqref{eq:pareto2}.  Therefore, we conclude that \method
	converges to Pareto optimal, indicating that it can naturally balance different objectives.
\end{proof}

%%%%% Experiment %%%%%
\section{Experiments} \label{sec:exp}
In the experiment, we use real-world industrial data to answer the following three research questions:
\begin{itemize}
	\item \textbf{RQ1:} How does \method perform compared with other baseline methods?
	\item \textbf{RQ2:} What is the best way to aggregate multiple objectives?
	\item \textbf{RQ3:} How does \method balance the exploration of different objectives?
\end{itemize}

\begin{table}[t]
	\centering
	\captionsetup{labelfont=bf}
	\begin{tabular}{c|cccc}
		\toprule
		\quad \textbf{Date}                  \quad   \quad                  & \quad\textbf{ \# of Ads}  \quad           & \quad\textbf{\# of clicks}  \quad    & \quad\textbf{\# of conversions}        \quad         \\
		\midrule
		\quad	\textbf{20190107}     \quad  \quad    & \quad 10,201    \quad              & \quad176,523,089   \quad                & \quad3,886,155   \quad                  \\
		\quad	\textbf{20190108} \quad \quad& \quad{10,416 }\quad & \quad{165,676,734}\quad & \quad{3,661,060} \quad\\
		\quad	\textbf{20190109} \quad  \quad                                     & \quad10,251    \quad             & \quad178,150,666 \quad                  & \quad3,656,714 \quad           \\
		\quad	\textbf{20190110}  \quad \quad                                      & \quad9,445  \quad              & \quad157,084,102   \quad                  &\quad 3,287,254  \quad                  \\
		\quad	\textbf{20190111}   \quad  \quad                                  & \quad10,035 \quad               &\quad 181,868,321   \quad                  &\quad 3,768,247 \quad        \\  
		\bottomrule
	\end{tabular}
	\vspace{1mm}
	\caption{Statistics of click data from Tencent bidding system.}
	\label{tb:statistics}
\end{table}

\subsection{Experiment Setup}\label{exp:setup}
\textbf{Dataset.} In the experiment, the dataset is collected from the real-time commercial ads bidding system of Tencent. There are nearly 10,000 ads daily with a huge volume of click and conversion logs. According to real-world business, the bidding interval is set to be 10 minutes (144 bidding sessions for a day), which is much shorter than one hour~\cite{jin2018real}. Basic statistics can be found in Table~\ref{tb:statistics}. 

\noindent\textbf{Compared Baselines.} 
We carefully select related methods for comparison and adopt the same settings for all the compared methods with 200 iterations. Details about implementation can be seen in Appendix~\ref{app:add2}.
\begin{itemize}
	\item \textbf{Proportional-Integral-Derivative (PID):}~\cite{bennett1993development} is a widely used feedback control policy, which produces the control signal from a linear combination of proportional, integral, and derivative factors. 
	\item \textbf{Advantage Actor-Critic (A2C):}~\cite{mnih2016asynchronous} makes the training process more stable by introducing an advantage function. ~\cite{jin2018real} generalizes the actor-critic structure in the RTB setting. 
	\item \textbf{Deep Q-Network (DQN):}~\cite{wang2017ladder} uses DQN with a single objective under the assumption of consistent state transition in the RTB problem, while the similar structure can also be coupled with a dynamic programming approrach~\cite{cai2017real}. 
	\item \textbf{Aggregated A3C (Agg-A3C):} Agg-A3C~\cite{mnih2016asynchronous} is proposed to disrupt the correlation of training data by introducing an asynchronous update mechanism. 
\end{itemize}
We linearly combine multiple rewards (following \emph{Reward Combination}) for all the baselines. Besides, we adopt two variants of our model: \emph{Objective1-A3C (O1-A3C)} and \emph{Objective2-A3C (O2-A3C)}, by only considering one of the objectives. We use four days of data for training and another day for testing and then use the cross-validation strategy on the training set for hyper-parameter selection. Similar settings can be found in literature~\cite{wang2017ladder,zhu2017optimized}.

\noindent\textbf{Evaluation Metrics.} We clarify the objectives of our problem based on the collected data. In Sec.~\ref{sec:aac}, we claim that our two objectives are: (1) \emph{minimize overall CPA}; (2) \emph{maximize conversions}. We refer to the industrial convention and redefine our goals in the experiments. \emph{Revenue} is a common indicator for platform earnings, which turns out to be proportional to conversions. \emph{Cost} is the money paid by advertisers, which also appears to be a widely accepted factor in online advertising. Therefore, without loss of generality, we reclaim our two objectives to be: 
\begin{equation}
	\text{Revenue}^{(j)}=\text {conversions }^{(j)}\times \text{CPA}_{\text{target}}^{(j)},\quad \text{Cost}^{(j)}=\text{\#clicks}^{(j)} \times \text{CPC}_{\text {next }}^{(j)}, 
\end{equation}
\begin{equation}
	\max \textbf{ROI} \leftarrow \max \sum_{Ad_{j}\in A} \frac{\text{Revenue}^{(j)}}{\text{Cost}^{(j)}}, 
\end{equation}
which corresponds to the first objective: \emph{CPA goal}, and  
\begin{equation}
	\max \textbf{Revenue}\leftarrow \max\sum_{Ad_{j}\in A} \text{Revenue}^{(j)}, 
\end{equation}
related to the second objective: \emph{Conversion goal}. 

For the two variants of \method, O1-A3C corresponds to maximizing ROI, while O2-A3C is related to maximizing Revenue. In addition to directly comparing these two metrics, we also use \emph{R-score} proposed in~\cite{lu} to evaluate the model performance. The higher the \emph{R-score}, the more satisfactory the advertisers and platform will be. In the real-world online ad system, PID is currently used to control bidding. We employ it as a standard baseline, and most of the comparison results will be based on PID, i.e., $value \rightarrow \frac{value}{value_{PID}}$, except for Sec.~\ref{sub:q3}.

\begin{table}[t!]
	\centering
	\captionsetup{labelfont=bf}
	\begin{tabular}{c|cccc}
		\toprule
		\quad \textbf{Model}\quad \quad&
		\quad\textbf{Relative Cost}\quad &
		\quad\textbf{Relative ROI}\quad &
		\quad\textbf{Relative Revenue}\quad &
		\quad\textbf{R-score}\quad \\
		\midrule
		\quad\textbf{PID} \quad\quad& \quad 1.0000 \quad & \quad{1.0000}\quad & \quad 1.0000 \quad& \quad1.0000\quad\\
		\quad\textbf{A2C}\quad\quad &  \quad 1.0366 (+3.66\%) \quad &  \quad0.9665 (-3.35\%) \quad&   \quad1.0019 (+0.19\%) \quad& \quad 0.9742\quad \\
		\quad\textbf{DQN}\quad\quad &  \quad  0.9765 (-2.35\%) \quad &  \quad1.0076 (+0.76\%) \quad& \quad0.9840 (-1.60\%) \quad& \quad 0.9966\quad\\
		\quad\textbf{Agg-A3C}\quad\quad &  \quad  1.0952 (+9.52\%) \quad & \quad0.9802 (-1.98\%) \quad &\quad 1.0625 (+6.25\%) \quad& \quad 0.9929\quad\\
		\midrule
		\quad\textbf{O1-A3C} \quad\quad&  \quad  0.9580 (-4.20\%) \quad&  \quad1.0170 (+1.70\%)\quad & \quad 0.9744 (-2.56\%)\quad& \quad 1.0070\quad\\
		\quad\textbf{O2-A3C} \quad\quad&  \quad  1.0891 (+8.91\%) \quad&  \quad0.9774 (-2.26\%)\quad & \quad 1.0645 (+6.45\%) \quad& \quad 0.9893\quad\\
		\midrule
		\quad\method \quad\quad&  \quad  1.0150 (+1.50\%) \quad&  \quad1.0267 (+2.67\%)\quad & \quad 1.0421  (+4.21\%) \quad & \quad 1.0203\quad\\
		\bottomrule
	\end{tabular}
	\vspace{2mm}
	\caption{Comparative results based on PID.}
	\label{tb:general}
\end{table}

\subsection{RQ1: Comparison with Recent Baselines}
We perform the comparison of \method with other approaches. The results are shown in Table~\ref{tb:general}. The values in the parentheses represent the percentage of improvement or reduction towards PID. An optimal method is expected to improve both metrics (ROI \& Revenue) compared with the current PID baseline.

\subsubsection{Objective Comparison.}
We find that \method best balances the trade-off between two objectives (ROI \& Revenue) based on the above considerations. Also, it has the highest R-score. Specifically, A2C is the worst since it gains a similar revenue (conversion goal) but a much lower ROI (CPA goal) than PID. The result proves that the A2C structure cannot fully capture the dynamics in the RTB environment. Based on a hybrid reward, DQN has a similar performance as O1-A3C, with relatively fewer conversions than other methods. We suspect the discrete action space may limit the policy to a local and unstable optimal. By solely applying the weighted sum in a standard A3C (Agg-A3C), the poor result towards ROI is not surprising. As the RTB environment varies continuously, fixing the formula of reward aggregation cannot capture the dynamic changes. It should be pointed out that two ablation models, O1-A3C and O2-A3C, present two extreme situations. O1-A3C performs well in the first ROI objective but performs poorly for the Revenue goal and vice versa for O2-A3C. By shifting the priority of different objectives over time, our proposed \method uses the agent's prior as a reference to make the decision in the future, precisely capturing the dynamics of the RTB sequence. Therefore, it outperforms all the other baselines.

Comparing \emph{Reward Partition} and \emph{Reward Combination}, the advantages of \method over other baselines show that our proposed method of accumulating rewards overall reduces the difficulty of agent learning and makes it easier for the policy network to converge around the optimal value.

\subsubsection{Bidding Quality Analysis.}
To further verify the superiority of \method compared to other methods, we analyze the relative bidding quality of these methods over PID. We group all the ads into five categories based on their bidding results. The detailed evaluation metrics can be found in Appendix~\ref{app:metric}.
\begin{figure}[t]
	\captionsetup{labelfont=bf}
	\centering
	\includegraphics[width=3.3in]{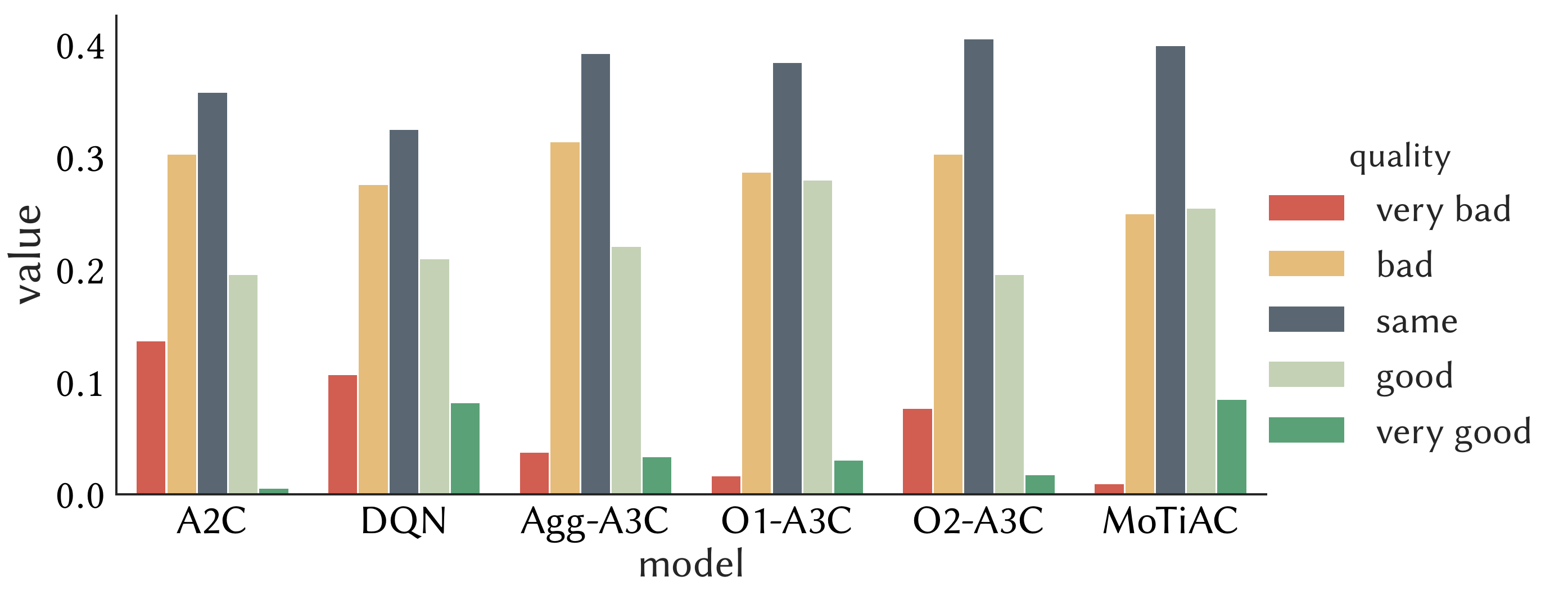}
	% \vspace{-2mm}
	\caption{Bidding quality distribution of compared methods over PID.}
	\label{fig:quality}
	% \vspace{-3mm}
\end{figure}
As shown in Fig.~\ref{fig:quality}, both A2C and O2-A3C present more bad results compared than good ones, indicating that these two models could not provide a gain for the existing bidding system at a finer granularity. O1-A3C has a relatively similar performance as PID, as they both aims at minimizing real \emph{CPA}. We also find that DQN tends to make bidding towards either very good or very bad, once again demonstrating the instability of the method. Agg-A3C shares the same distribution pattern with O1-A3C and vanilla PID, which indicates that the combined reward does not work in our scenario. The proposed \method turns out to have a desirable improvement over PID with more ads on the right \emph{good} side and fewer ads on the left \emph{bad} side. Note that the negative transfer of multi-objective tasks makes some bidding results inevitably worse. However, we can still consider that \method can achieve the best balance among all the compared methods.

\begin{figure}[t]	
	\centering
	\captionsetup{labelfont=bf}
	\includegraphics[width=4.2in]{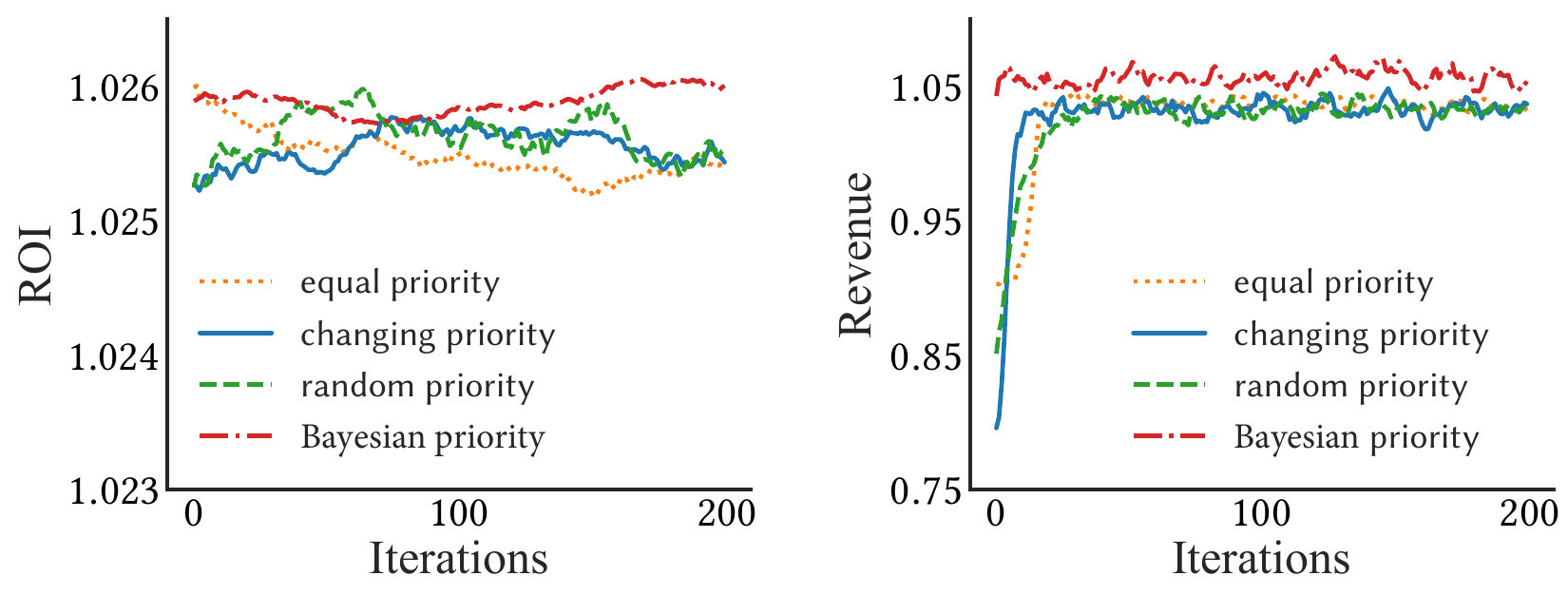}
	\caption{Results under different priority functions.}
	\label{fig:lambda}
\end{figure}

\subsection{RQ2: Variants of $\boldmath w_k$}\label{sub:q2}
To give a comprehensive view of \method, we perform different
ways to aggregate objectives. Four different variants of  $w_k$ are considered in the experiment. Since we have two objectives, we use $w_1(t)$ for the first objective and $1-w_1(t)$ for the second:
\begin{itemize}
	\item equal priority: $w_1(t) = \frac 12$;
	\item changing priority: $w_1(t)=
	\exp(-\alpha\cdot t)$ with a scalar $\alpha$;
	\item random priority: $w_1(t) = \text{random}([0,1])$;
	\item Bayesian
	priority: One can refer to Eqn.~\eqref{eq:phi}.
\end{itemize}
As shown in Fig.~\ref{fig:lambda}, we present the training curves for ROI and Revenue. The first three strategies are designed before training and will not adjust to the changing environment. It turns out that they perform similarly in both objectives and could gain a decent improvement over the PID case by around +2.5\% in ROI and +3\% in Revenue. However, in \emph{equal priority}, the curve of ROI generally drops when the iteration goes up, which stems from the fact that fixed equal weights cannot fit the dynamic environment. For \emph{changing priority}, it is interesting that ROI first increases then decreases for priority shifting, as different priority leads to different optimal. In \emph{random priority}, curves dramatically change in a small range since the priority function outputs the weight randomly. The \emph{Bayesian priority} case, on the contrary, sets priority based on the conformity of the agent's prior and current state. Reward partition with agent prior dominates the first three strategies by an increasingly higher ROI achievement by  +2.7\% and better Revenue by around +4.2\%.

\subsection{RQ3: Case Study}\label{sub:q3}

\begin{figure}[t]
	\centering
	\captionsetup{labelfont=bf}
	\includegraphics[width=3.7in]{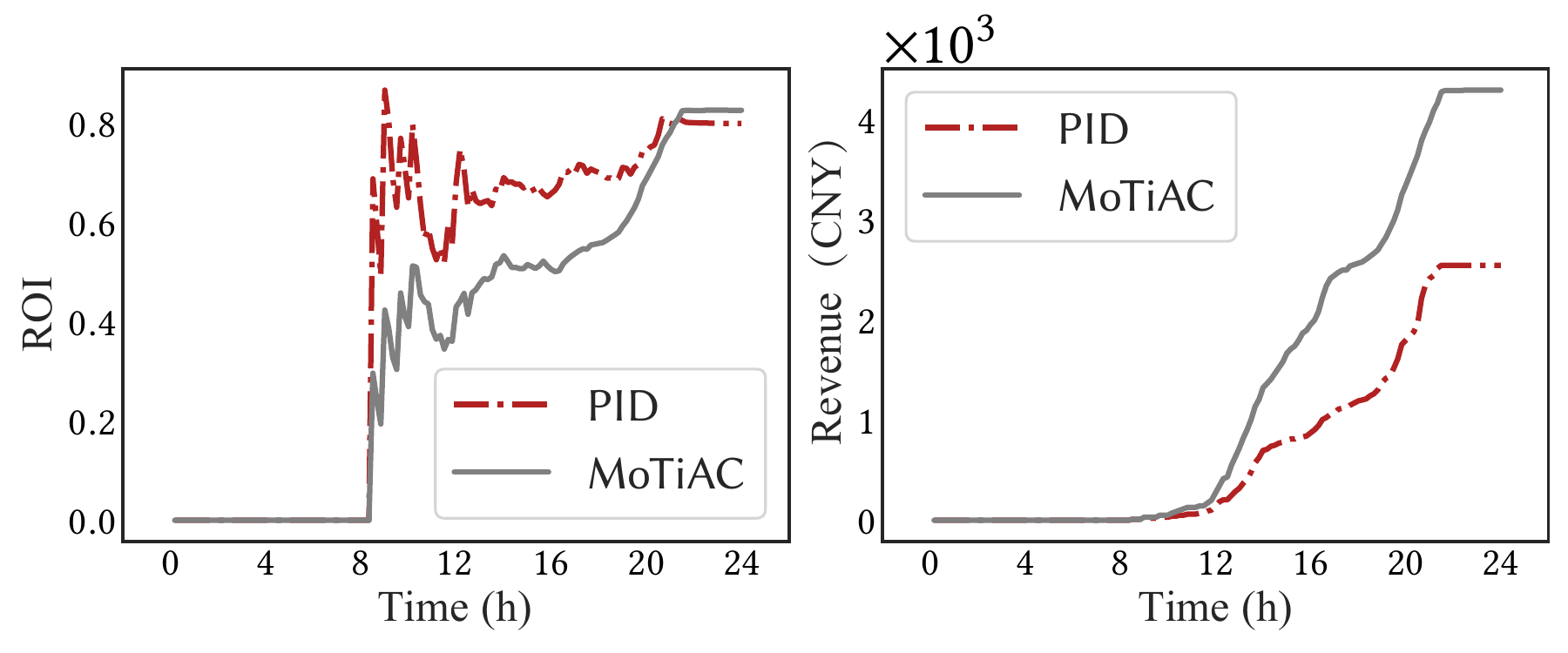}
	\caption{ROI and Revenue curves of the target ad's reponse.}
	\label{fig:caseStudy2}
\end{figure}

\begin{table}[t]
	\centering
	\captionsetup{labelfont=bf}
	\begin{tabular}{|c|ccc|}
		\toprule
		\quad \textbf{Models} \quad & \quad\textbf{Revenue (CNY)} \quad& \quad\textbf{Cost (CNY)} \quad& \quad\textbf{ROI} \quad\\
		\midrule
		\quad\textbf{PID} \quad\quad&\quad $3.184\times 10^3$\quad & \quad$2.548\times 10^3$ \quad & \quad0.8003 \quad\\
		\quad\method \quad\quad& \quad$4.298\times 10^3$ \quad& \quad$5.199\times 10^3$ \quad & \quad0.8267\quad \\
		\bottomrule
	\end{tabular}
	\vspace{2mm}
	\caption{Numerical results of the target ad using PID and \method.}
	\label{tb:twoAds}
\end{table}
In this section, we try to investigate how \method balances the exploration of multiple objectives and achieves the optimal globally. We choose one typical ad with large conversions and show the bidding process within 24 hours. As PID is the current model in the real ad system, we use PID to compare with \method and draw the results of ROI and Revenue curve in Fig.~\ref{fig:caseStudy2}. We also collect the final numerical results in Table~\ref{tb:twoAds}. 

Fig.~\ref{fig:caseStudy2} shows a pretty low ROI initially. For the target ad, both models first try to lift the ROI. Based on the figure presented on the left, the red dashed curve rises from 0 to about 0.7 sharply for PID at 8h. The potential process should be that PID has given up most of the bid chances and only concentrates on those with a high conversion rate (CVR) so that we have witnessed a low Revenue gain of the PID model in the right figure from 8h to around 21h. Though the ROI curve remains relatively low, our \method can select good impression-level chances while considering the other objective. At 24h, \method finally surpasses PID in ROI because of the high volume of pre-gained Revenue. With long-term consideration, \method beats PID on both the cumulative ROI and Revenue. We can conclude that PID is greedy out of the immediate feedback mechanism. It is always concerned with the current situation and never considers further benefits. When the current state is under control, PID will appear conservative and give a shortsighted strategy, resulting in a seemingly good ROI and poor Revenue (like the red curve in Fig.~\ref{fig:caseStudy2}). However, \method has a better overall view. It foresees the long-run benefit and will keep exploration even temporarily deviating from the right direction or slowing down the rising pace (ROI curve for the target ad at 8h). Under a global overview, \method can finally reach better ROI and Revenue than PID.

%%%%% RelatedWork %%%%%
%\vspace{-3mm}
\section{Related Work} \label{sec:rw}

\vspace{0.5mm}
\textbf{Real-time Bidding.} Researchers have proposed static methods~\cite{Perlich2012} for optimal biddings, such as constraint optimization~\cite{zhang2014optimal}, to perform an impression-level evaluation. However, traditional methods inevitably ignore that real-world situations in RTB are often dynamic~\cite{wu2018budget} due to the unpredictability of user behavior~\cite{jin2018real} and different marketing plans~\cite{xu2015smart} from advertisers. Furthermore, the auction process of optimal bidding is formulated as a Markov decision process (MDP) in recent study~\cite{jin2018real,lu}. Considering the various goals of different players in RTB, a robust framework is required to balance these multiple objectives. Therefore, we are motivated to propose a novel multi-objective RL model to maximize the overall utility of RTB. 

\medskip
\noindent\textbf{Reinforcement Learning.} Significant achievements have been made by the emergence of RL algorithms, such as policy gradient \cite{sutton2000policy} and actor-critic~\cite{konda2000actor}. With the advancement of GPU and deep learning (DL), more successfully deep RL algorithms~\cite{lillicrap2015continuous,mnih2016asynchronous} have been proposed and applied to various domains. Meanwhile, there are previous attempts to address the multi-objective reinforcement learning (MORL) problem~\cite{Hayes2021}, where the objectives are combined mainly by static or adaptive linear weights~\cite{AbelsRLNS19,pasunuru2018multi} or captured by a set of policies and evolving preferences~\cite{pirotta2015multi}.

%%%%% Conclusion %%%%%
\section{Conclusion and Future Directions}
In this paper, we propose \textbf{M}ulti-\textbf{O}bjec\textbf{Ti}ve \textbf{A}ctor-\textbf{C}ritics for real-time bidding in display advertising. \method utilizes objective-aware actor-critics to solve the problem of multi-objective bidding optimization. Our model can follow adaptive strategies in a dynamic RTB environment and outputs the optimal bidding policy by learning priors from historical data. We conduct extensive experiments on the real-world industrial dataset. Empirical results show that \method achieves state-of-the-art on the Tencent advertising dataset. One future direction could be extending multi-objective solutions with priors in the multi-agent reinforcement learning area.

\bibliographystyle{splncs04}
\bibliography{ecml_motiac.bib}

\appendix
\newpage
%%%%% Additional Material %%%%%
\section{Supplementary Material}\label{add1}

\subsection{Pricing Evaluation Metrics}\label{app:metric}
We present our pricing evaluation metric below: Algorithm~\ref{algo:pricing} is the pricing evaluation strategy based on our business, and in
Algorithm~\ref{algo:score} we specify the \emph{Score} function.

\begin{algorithm}[htbp!]
	\SetAlgoLined
	// We use $Value_{method}$ to denote the value given by other approaches, and $Value_{PID}$ denotes the value given by PID\;
	// Pricing score is given per ad\;
	\For{$ad_j \in A$}{
		// $CPA$ calculation\;
		$\Delta cost = cost_{method} - cost_{PID}$\;
		$\Delta conversion = conversion_{method} - conversion_{PID}$\;
		$\Delta CPA = (\Delta cost+0.5\times CPA_{target})~/~ (\Delta conversion+0.5)$\;
		$CPA_{PID} = cost_{PID}~/~(conversion_{PID}+0.01) $\;
		$CPA_{method} = cost_{method}~/~(conversion_{method}+0.01) $\;
		// Bias calculation\;
		$\Delta bias = \Delta CPA~/~CPA_{target} - 1$\;
		$bias_{PID} = CPA_{PID}~/~CPA_{target} - 1$\;
		$bias_{method} = CPA_{method}~/~CPA_{target} - 1$\;
		// Enter the loop\;
		\uIf{$bias_{PID}\leq0.2$}{
			$\Delta loss = \Delta cost~/~cost_{PID}$\;
			$score = Score(\Delta loss)$\;
		}\Else{
			\If{$bias_{method}\leq 0.2 \bigwedge bias_{PID} \geq 0.3$}{
				$score = 2$\;
			}
			$\Delta loss = \Delta cost~/~cost_{PID}$\;
			\uIf{$\Delta bias \geq 0.3$}{
				$score = -Score(\Delta loss)$\;
			}\Else{
				$score = Score(\Delta loss)$\;
			}
		}
		Return score for $ad_j$\;
	}
	
	\caption{Pricing Evaluation over PID.}
	\label{algo:pricing}
\end{algorithm}

\begin{algorithm}[htbp!]
	\SetAlgoLined
	\DontPrintSemicolon
	\KwIn{$ \Delta loss $}
	\KwOut{$ score $ }
	
	\SetKwFunction{FMain}{\textrm{\textit{Score}}}
	\SetKwProg{Fn}{function}{:}{}
	\Fn{\FMain{$\Delta loss$}}{
		\uIf{$ \Delta loss \leq 0.1$}{Return 0 // similar performance\;}\uElseIf{$ \Delta loss \leq 0.2$}{Return 1 // good / bad performance\;}\Else{Return 2 // very good / bad performance\;}// sign will be given outside
	}
	\textbf{end}
	\caption{Score Function}
	\label{algo:score}
\end{algorithm}

\subsection{Implementation Details}\label{app:add2}
We shall present our implementation details for all baselines and our \method as follows:
\begin{itemize}
	\item \textbf{PID}: control signal is given by difference across two intervals in our experiment, in order to making sure less fluctuation and stable result. Ratio for proportional factor is 0.04 , for the integral factor is 0.25 and for the derivative factor is 0.002. 
	\item \textbf{A2C, Agg-A3C}: in A2C, we use two hidden layers with 256 and 128 dimensions. Output for \emph{Actor} is one-dimension scalar $\mu$ and $\sigma$ (Gaussian Distribution), and then actions will be sampled from the above-mentioned distribution with boundaries. Output for \emph{Critic} is one-dimension scalar. \emph{State} is a normalized 17-dimension feature vector. 
	Agg-A3C has the same architecture, but is on extension of A2C with 5 workers running in parallel and updates global net asynchronously. \item {\textbf{DQN}}: all settings in DQN are the same as in A2C and Agg-A3C, except for \emph{Actor}. The output of \emph{Actor} has 100 dimension w.r.t 100 values ranging from 0 to 1, and 0.01 as its interval. Then based on values we choose the best $r \in [0,1]$, and gives $r\times \text{CPA}_{target}$ as $\text{CPC}_{bid}$.
	\item {\textbf{O1-A3C, O2-A3C}}: all settings in O1-A3C and O2-A3C are the same as in A2C and Agg-A3C, except for reward. \emph{Reward} in O1-A3C and O2-A3C is set according to their objectives.
	\item {\textbf{\method}}: we use two hidden layers with 256 and 128 dimensions for all nets in \method. However, the difference is that we have two critics in global nets with the same structures. And for each update period, one group of workers can only push their gradients to actor and one of the critics. We set 5 worker as an objective group. Also, \method is implemented with multiprocessing and shared memory.
\end{itemize}
\end{document}